\documentclass{article}
\usepackage{ijcai20}
\usepackage{times}
\usepackage{soul}
\usepackage{helvet}
\usepackage{courier}
\usepackage{booktabs} 
\usepackage{multirow}
\usepackage[english]{babel}
\usepackage[utf8]{inputenc}
\usepackage{mathrsfs}
\usepackage{graphicx}
\usepackage{amssymb}
\usepackage{url}
\usepackage{subfigure}
\usepackage{amsmath}
\usepackage{algorithmicx}
\usepackage{algorithm}
\usepackage{amsmath}
\usepackage{amsthm}
\usepackage{bbm}
\usepackage{algpseudocode} 
\usepackage{threeparttable}
\usepackage{arydshln}

\usepackage{color}
\newcommand{\mb}{\mathbf}
\newtheorem{thm}{\bf Theorem}
\newtheorem{lemma}{Lemma}
\newtheorem{defn}{Definition}

\renewcommand{\arraystretch}{0.5}

\newcommand{\our}{\text{DEAM}}
\newcommand{\ams}{\text{AMSGrad}}
\newcommand{\adadelta}{\text{AdaDelta}}
\newcommand{\rmsprop}{\text{RMSProp}}
\newcommand{\adam}{\textsc{ADAM}}
\newcommand{\adagrad}{\text{AdaGrad}}
\newcommand{\tabincell}[2]{\begin{tabular}{@{}#1@{}}#2\end{tabular}}
\begin{document}
	
\title{DEAM: Adaptive Momentum with Discriminative Weight for Stochastic Optimization}

\author{
Jiyang Bai\and
Yuxiang Ren\and
Jiawei Zhang
\affiliations
IFM Lab, Florida State University\\
\emails
\{jiyang, yuxiang, jiawei\}@ifmlab.org	
}

\maketitle

\begin{abstract}
Optimization algorithms with momentum, e.g., ({\adam}), have been widely used for building deep learning models due to the faster convergence rates compared with stochastic gradient descent (SGD). Momentum helps accelerate SGD in the relevant directions in parameter updating, which can minify the oscillations of parameters update route. However, there exist errors in some update steps in optimization algorithms with momentum like {\adam}. The fixed momentum weight (e.g., $\beta_1$ in {\adam}) will propagate errors in momentum computing.
In this paper, we introduce a novel optimization algorithm, namely \textbf{D}iscriminative w\textbf{E}ight on \textbf{A}daptive \textbf{M}omentum ({\our}). Instead of assigning the momentum term weight with a fixed hyperparameter, {\our} proposes to compute the momentum weight automatically based on the discriminative angle. 
In this way, {\our} involves fewer hyperparameters. {\our} also contains a novel backtrack term, which restricts redundant updates when the correction of the last step is needed. 
Extensive experiments demonstrate that {\our} can achieve a faster convergence rate than the existing optimization algorithms in training the deep learning models of both convex and non-convex situations.

\end{abstract}

\section{Introduction}\label{sec:introduction}
Deep learning methods can achieve outstanding performance in multiple fields including computer vision~\cite{KXSJ16}, natural language processing~\cite{LFMK15,DKY14}, speech and audio processing~\cite{MN17}, and graph analysis~\cite{JiaweiZhang2019Tutorial}. Training deep learning models involves an optimization process to find the parameters in the model that minimize the loss function. At the same time, the number of parameters commonly used in deep learning methods can be very huge.

Therefore, optimization algorithms are critical for deep learning methods: not only the model performance, but also training efficiency are greatly affected. 
In order to cope with the high computational complexity of training deep learning methods, stochastic gradient descent (SGD)~\cite{sgd_overview} is utilized to update parameters based on the gradient of each training sample instead. The idea of momentum~\cite{moment}, inspired by Newton's first law of motion, is used to handle the oscillations of SGD. SGD with momentum~\cite{IJGG13} achieves the faster convergence rate and better optimization results compared with the original SGD. In gradient descent based optimization, training efficiency is also greatly affected by the learning rate. 
AdaGrad~\cite{adagrad} is the first optimization algorithm with adaptive learning rates, which makes use of the learning rate decay. AdaDelta~\cite{adadelta} subsequently improves AdaGrad to avoid the extremely small learning rates. {\adam}~\cite{ADAM} involves both adaptive learning~\cite{sgd_overview} and momentum~\cite{moment} and utilizes the exponential decay rate $\beta_1$ (momentum weight) to accelerate the convergence in the relevant directions and dampen oscillations. However, the decay rate $\beta_1$ of the first-order momentum $\mb{m}_t$ in {\adam} is a fixed number, and the selection of the hyperparameter $\beta_1$ may affect the performance of {\adam} greatly. Commonly, $\beta_1 = 0.9$ is the most widely used parameter as introduced in~\cite{ADAM}, but there is still no theoretical evidence proving its advantages. 

During the optimization process, it is common that there exist errors in some update steps. These errors can be caused by the inappropriate momentum calculation, and then lead to slower convergence or oscillations. For each parameter updates, using the fixed momentum weight fails to take the different influence of the current gradient into consideration, which will render errors in momentum computing. For example, when there exist parts of opposite eigen components~\cite{moment} between the continuous two parameter updates (we regard this situation as an error), the current gradient should be assigned a larger weight to correct the momentum in the last update, instead of being placed with a fixed influence. 
We will illustrate this problem through cases in Section~\ref{subsubsec: beta_motivation} where {\adam} with a fixed weight $\beta_1$ cannot handle some simple but intuitive convex optimization problems. Based on this situation, we need to control the influence of momentum by an adaptive weight. What's more, designing hyperparameter-free optimization algorithms has been a very important research problem in recent years, controlling the number of hyperparameters will not only stabilize the performance of the optimization algorithm, but also release the workload of hyperparameters tuning. 

In this paper, we introduce a novel optimization algorithm, namely {\our} (\textbf{D}iscriminative w\textbf{E}ight on \textbf{A}daptive \textbf{M}omentum) to deal with the aforementioned problems. {\our} proposes an adaptive momentum weight $\beta_{1,t}$, which will be learned and updated in each training iteration automatically. 
Besides, {\our} employs a novel backtrack term $d_t$, which will restrict redundant updates when {\our} decides that the correction of the previous step is needed.
We also provide the theoretic analyses about the adaptive momentum weight along with extensive experiments. Based on them, we verify that the adaptive momentum term weight $\beta_{1,t}$ and the operation of backtrack term $d_t$ can be crucial for the performance of the learning algorithms.

Here, we summarize the detailed learning mechanism of {\our} as follows:
\begin{itemize}
	\item {\our} computes adaptive momentum weight $\beta_{1,t}$ based on the ``discriminative angle" $\theta$ between the historical momentum and the newly calculated gradient.
	\item {\our} introduces a novel backtrack term, i.e., $d_t$, which is proposed to correct the redundant update of the previous training epoch if it is necessary. 
	The calculation of $d_t$ is also based on the discriminative angle $\theta$.
	\item {\our} involves fewer hyperparameters than the {\adam} during the training process, which can decrease the workload of hyperparameter tuning.
\end{itemize}

Detailed information about the learning mechanism and the concepts mentioned above will be described in the following sections. This paper will be organized as follows. In Section 2, we will cover some related works about widely used optimization algorithms. In Section 3, we will analyze more detail of our proposed algorithm, whose theoretic convergence rate will also be studied. Extensive experiments will be exhibited in Section 4. Finally, we will give a conclusion of this paper in Section 5.
\section{Related Works} \label{sec:relatedwork}
\noindent\textbf{Stochastic Gradient Descent}: Stochastic gradient descent (SGD)~\cite{sgd_overview,bianchi2012convergence} performs parameter updating for each training example and label. The advantages of SGD include fast converging speed compared with gradient descent and preventing redundancy~\cite{sgd_overview}.~\cite{on_variance_reduction} use the variance reduction methods to accelerate the training process of SGD.

\noindent\textbf{Adaptive Learning Rates}: To overcome the problems brought by the unified learning rate, some variant algorithms applying adaptive learning rate~\cite{behera2006adaptive} have been proposed, such as AdaGrad~\cite{adagrad}, {\adadelta}~\cite{adadelta}, RMSProp~\cite{rmsprop}, {\adam}~\cite{ADAM} and recent ESGD~\cite{dauphin2015equilibrated}, AdaBound~\cite{luo2019adaptive}. {\adagrad} adopts different learning rates to different variables. One drawback of AdaGrad is that with the increasing of iteration number $t$, the adaptive term may inflate continuously, which leads to a very slow convergence rate in the later stage of the training process. {\rmsprop} can solve this problem by using the moving average of historical gradients.

\noindent\textbf{Momentum}: Momentum~\cite{moment,impor_mom_deeplearning,li2017convergence,dozat2016incorporating,mitliagkas2016asynchrony} is a method that helps accelerate SGD in the relevant direction and prevent oscillations on the descent route. The momentum accelerates updates for dimensions whose gradients are in the same direction as historical gradients, and decelerates updates for dimensions whose gradients are the opposite. Momentum is also applied to Nesterov accelerated gradient (NAG)~\cite{nag}, which provides the momentum term with the estimating next position capability. Instead of using the current location to calculate the gradient, NAG first approximates the next position of the variables, then uses the approximated future position to compute the gradient.
{\adam}~\cite{ADAM,normalized_direction} is proposed based on momentum and adaptive learning rates for different variables. {\adam} records the first-order momentum and the second-order momentum of the gradients using the moving average, and further computes the bias-corrected version of them. Based on {\adam}, Keskar et al.~\cite{switch_from_adam_sgd} proposes to switch from {\adam} to SGD during the training process. {\ams}~\cite{amsgrad} is a modified version of {\adam}, which redefines second-order momentum by a maximum function. 


\section{Proposed Algorithm}\label{sec:method} 
Our proposed algorithm {\our} is presented in Algorithm~\ref{alg:cwm}. In the algorithm, $f_1, f_2,\dots, f_T$ is a sequence of loss functions computed with the training mini-batches in different iterations (or epochs). 
\begin{algorithm}[tb]
	\footnotesize
	\caption{{\our} Algorithm} 
	\hspace*{0.02in} {\bf Input:} 
	loss function $f(\mb{w})$ with parameters $\mb{w}$; learning rate $\{\eta_t\}^T_{t=1}$; $\beta_2 = 0.999$\\
	\hspace*{0.02in} {\bf Output:} trained parameters
	\begin{algorithmic}[1]
		\State $\mb{m}_0 \gets \mb{0}$\enspace /* Initialize first-order momentum */
		\State $\mb{v}_0 \gets \mb{0}$, $\mb{\hat{v}}_0 \gets \mb{0}$\enspace /* Initialize second-order momentum */
		\For {$t = 1,2,\dots , T$}
		\State $\mb{g}_t = \nabla f_t(\mb{w}_t)$,
		\vspace{0.02in}
		\State $\theta = \left\langle\frac{\mb{m}_{t-1}}{\sqrt{\mb{\hat{v}}_{t-1}}}, \mb{g}_t\right\rangle$ \enspace /* The operator $\left\langle\cdot ,\cdot \right\rangle$ represents the angle between two vectors. */
		\If {$\theta \in [0,\frac{\pi}{2})$}
		\State $\beta_{1,t} = \sin\theta / K + \epsilon$
		\Else
		\State $\beta_{1,t} = 1 / K$ \enspace /* Here, $K = \frac{10(2+\pi)}{2\pi}$. */
		\EndIf
		\vspace{0.02in}
		\State $\mb{m}_t = (1-\beta_{1,t})\cdot\mb{m}_{t-1} + \beta_{1,t}\cdot\mb{g}_t$
		\vspace{0.02in}
		\State $\mb{v}_t = \beta_2\cdot\mb{v}_{t-1} + (1-\beta_2)\cdot\mb{g}_t\odot\mb{g}_t$\enspace /* $\odot$ is element-wise multiplication.*/
		\vspace{0.02in}
		\State $\mb{\hat{v}}_t = \max\{\mb{\hat{v}}_{t-1}, \mb{v}_t\}$
		\State $d_t = \min\{0.5\cos\theta, 0\}$
		\vspace{0.02in}
		\State $\mb{\Delta}_t = d_t\cdot\mb{\Delta}_{t-1} - \eta_t\cdot\frac{\mb{m}_t}{\sqrt{\mb{\hat{v}}_t}}$
		\State $\mb{w}_t = \mb{w}_{t-1} + \mb{\Delta}_t$
		\EndFor 
		
		\State \Return $\mb{w}_T$
	\end{algorithmic}\label{alg:cwm}
\end{algorithm}
{\our} introduces two new terms in the learning process: (1) the adaptive momentum weight $\beta_{1,t}$
, and (2) the``backtrack term" $d_{t}$. In the $t_{th}$ training iteration, both $\beta_{1,t}$ and $d_t$ are calculated based on the ``discriminative angle" $\theta$, which is the angle between previous $\mb{m}_{t-1}/\sqrt{\mb{\hat{v}}_{t-1}}$ and current gradient $\mb{g}_t$ (since essentially both $\mb{m}_{t-1}/\sqrt{\mb{\hat{v}}_{t-1}}$ and $\mb{g}_t$ are vectors, there exists an angle between them). Here, $\mb{m}$ is the first-order momentum that records the exponential moving average of historical gradients; $\mb{v}$ is the exponential moving average of the squared gradients, which is called the second-order momentum. In the following parts of this paper, we will denote $\mb{m}_{t-1}/\sqrt{\mb{\hat{v}}_{t-1}}$ as the ``update volume'' in the ${(t-1)}_{th}$ iteration. 
Formally, $\beta_{1,t}$ determines the weights of previous first-order momentum $\mb{m}_{t-1}$ and current gradient $\mb{g}_t$ when calculation the present $\mb{m}_t$. Meanwhile, the backtrack term $d_t$ represents the returning step of the previous update on parameters. We can notice that in each iteration, after the $\theta$ has been calculated, the $\beta_{1,t}$ and $d_t$ are directly obtained according to the $\theta$. 
In this way, we can calculate appropriate $\beta_{1,t}$ as the discriminative angle changes. The $d_t$ term balances between the historical update term $\mb{\Delta}_{t-1}$ (defiend in Algorithm~\ref{alg:cwm}) and the current update volume $\mb{m}_t /\sqrt{\mb{\hat{v}}_t}$ when computing $\mb{\Delta}_t$. In the proposed {\our}, $\beta_{1,t}$ and $d_t$ terms can collaborate with each other and achieve faster convergence.

\subsection{Adaptive Momentum Weight $\beta_{1,t}$}
\subsubsection{Motivation}\label{subsubsec: beta_motivation}
In the {\adam}~\cite{ADAM} paper, (the first-order) momentum's weight (i.e., $\beta_1$) is a pre-specified fixed value, and commonly $\beta_1 = 0.9$. It has been used in many applications and the performance can usually meet the expectations. However, this setting is not applicable in some situations. For example, for the case 
\begin{equation}
f(x,y) = x^2 + 4y^2 ,
\end{equation}
where $x$ and $y$ are two variables, it is obvious that $f$ is a convex function. If $f(x,y)$ is the objective function to optimize, we try to use {\adam} to find its global optima.
\begin{figure}[t]
	\vspace{-20pt}
	\centering
	\includegraphics[width=6cm,height=3.75cm]{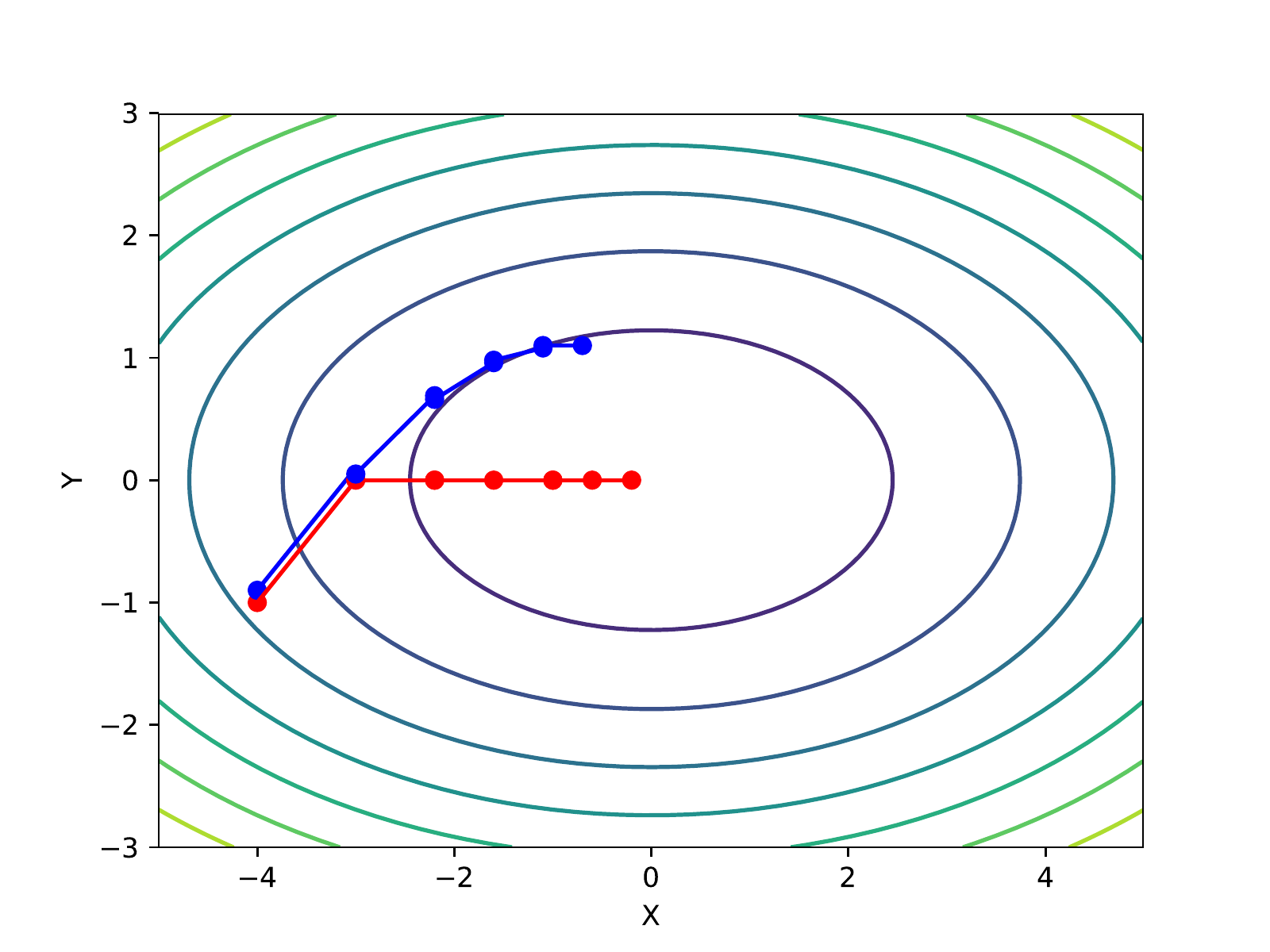}
	\caption{The update routes of {\adam} with $\beta_1 = 0.9$ (the blue line) and $\beta_1 = 0.0$ (the red line).}
	\label{fig:counter-example}
\end{figure}

\begin{figure*}[t]
	\centering
	\subfigure[The "zig-zag" route]{\label{fig:zigzag}
		\begin{minipage}[l]{0.6\columnwidth}
			\centering
			\includegraphics[width=0.8\textwidth,height=3cm]{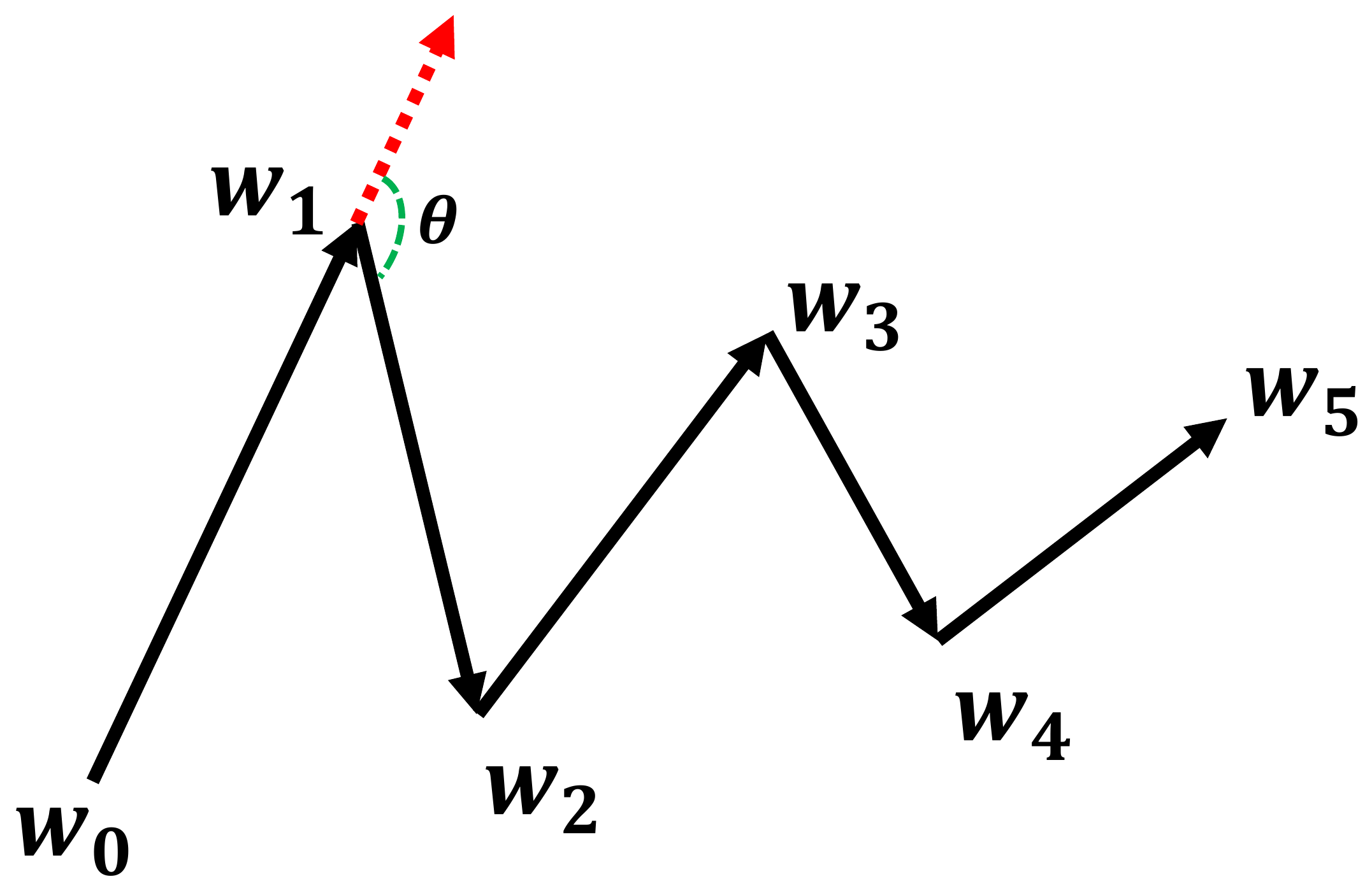}
		\end{minipage}
	}
	\subfigure[Axis decomposition]{\label{fig:axis_decomp}
		\begin{minipage}[l]{0.6\columnwidth}
			\centering
			\includegraphics[width=0.8\textwidth,height=3cm]{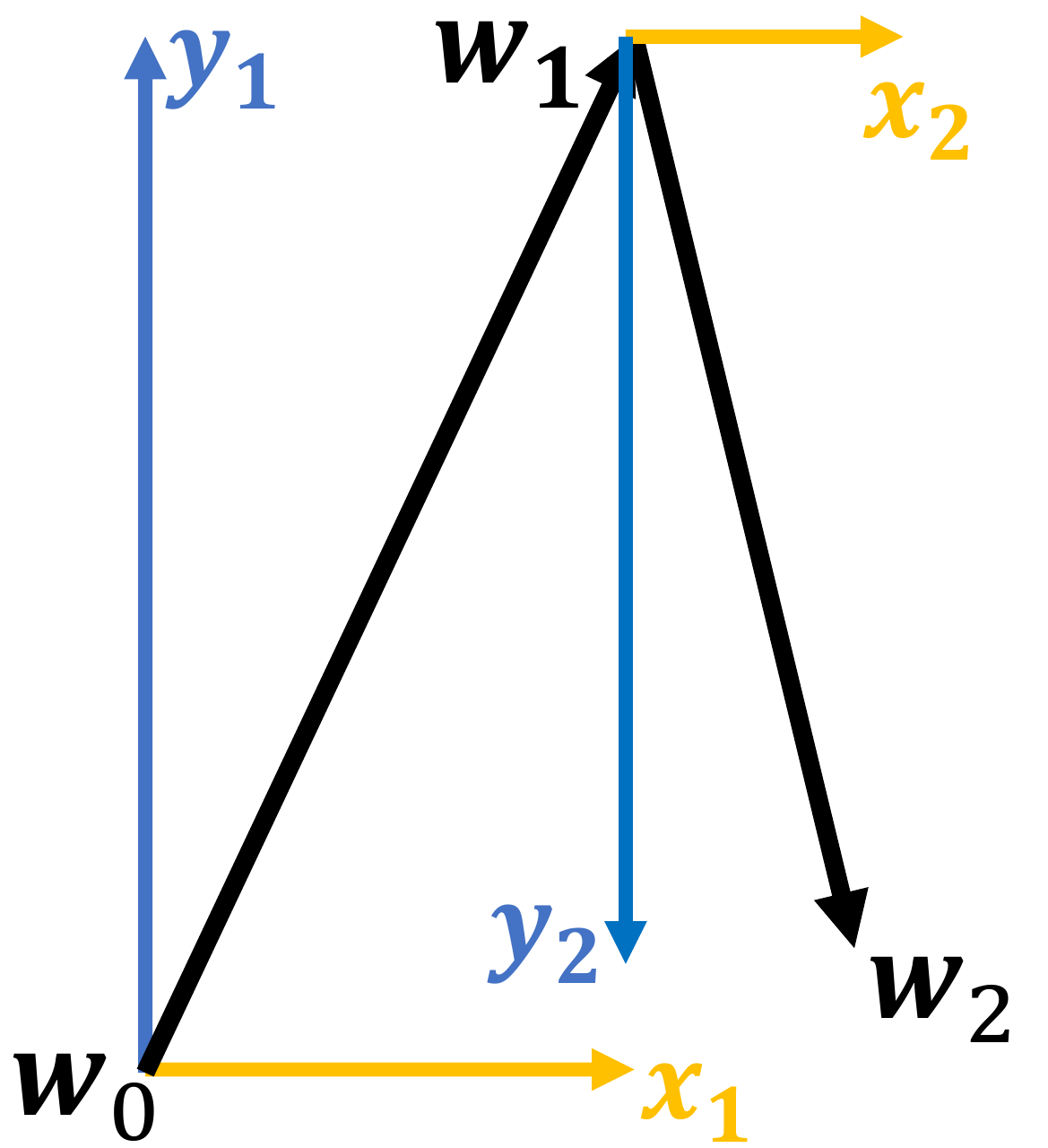}
		\end{minipage}
	}
	\subfigure[Example: When $\theta = 180^\circ$]{\label{fig:theta180}
		\begin{minipage}[l]{0.6\columnwidth}
			\centering
			\includegraphics[width=0.8\textwidth,height=3cm]{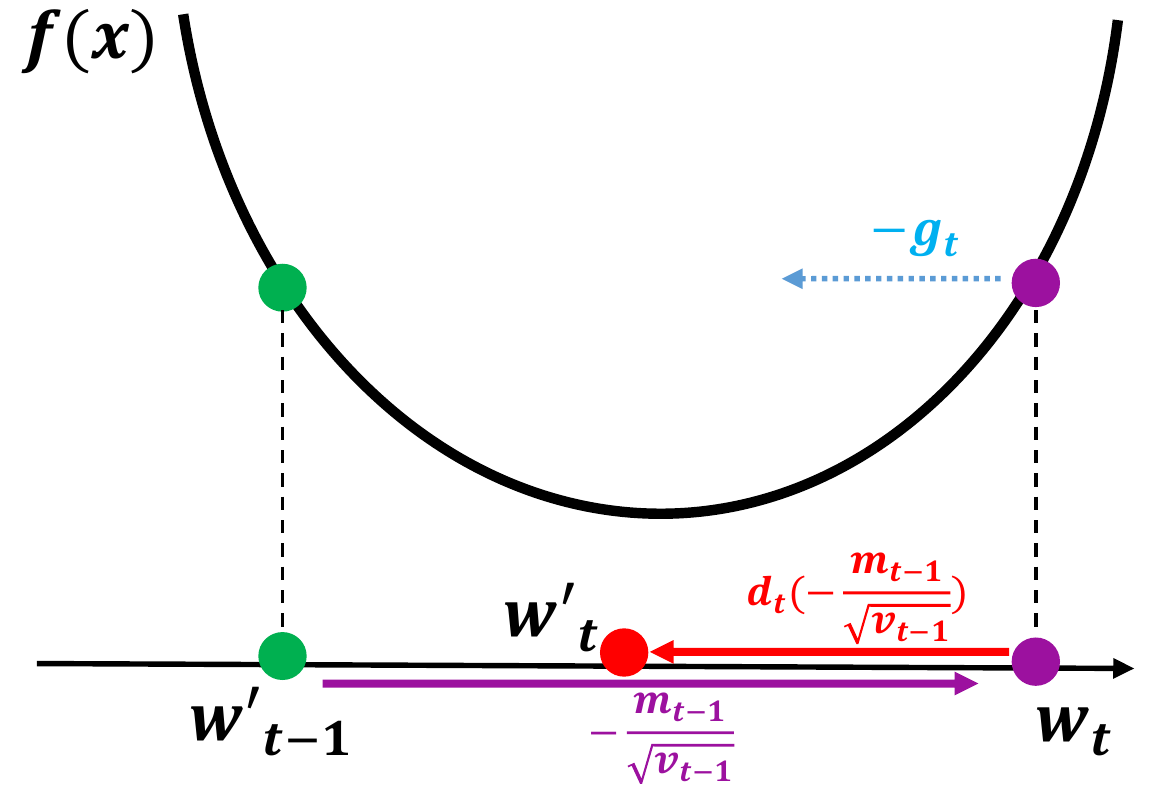}
		\end{minipage}
	}
	\caption{Some examples about $d_t$}
\end{figure*}
Let's assume ADAM starts the variable search from $(-4,-1)$ (i.e., the initial variable vector is $\mb{w}_0 = (-4,-1)^\top$) and the initial learning rate is $\eta_1 = 1$. Different choices of $\beta_1$ will lead to very different performance of ADAM. For instance, in Figure~\ref{fig:counter-example}, we illustrate the update routes of ADAM with $\beta_1 = 0.9$ and $\beta_1 = 0.0$ as the blue and red lines, respectively. In Figure~\ref{fig:counter-example}, the ellipse lines are the contour lines of $f(x,y)$, and points on the same line share the same function value. We can observe that after the first updating, both of the two approaches will update variables to $(-3, 0)$ point (i.e., the updated variable vector will be $\mb{w}_1 = (-3, 0)^\top$). In the second step, since the current gradient $\mb{g}_2 = (-6, 0)^\top$, the {\adam} with $\beta_1 = 0.0$ will update variables in the $(1, 0)$ direction. Meanwhile, for the {\adam} with $\beta_1 = 0.9$, its $\mb{m}_2 $ is computed by integrating $\mb{m}_1$ and $\mb{g}_2$ together (whose weights are $\beta_1$ and $1-\beta_1$, respectively). Therefore the updating direction of it will be more inclined to the previous direction instead. Compared with ADAM with $\beta_1 = 0.0$, the ADAM with $\beta_1 = 0.9$ takes much more iterations until converging.

From the analysis above, we can observe that, a careful tuning and updating of $\beta_1$ in the learning process can be crucial for the performance of {\adam}. However, by this context so far, there still exist no effective approaches for guiding the parameter tuning yet. To deal with this problem, {\our} introduces the concept of discriminative angle $\theta$ for computing $\beta_1$ automatically as follows.

\subsubsection{Mechanism}
The momentum weight $\beta_1$ will be updated in each iteration in {\our}, and we can denote its value computed in the $t_{th}$ iteration as $\beta_{1,t}$ formally.
Essentially, in the $t_{th}$ iteration of the training process, both the previous update volume 
and $\mb{g}_t$ are vectors (or directions), and these directions directly decide the updating process. Thus we try to extract their relation with the help of angle, and subsequently determine the weight $\beta_{1,t}$ (or $1 - \beta_{1,t}$) by the angle. 

In Algorithm~\ref{alg:cwm}, the discriminative angle $\theta$ in the $t_{th}$ iteration is calculated by
\begin{small}
	\begin{equation}
	\theta = \left\langle -\frac{\mb{m}_{t-1}}{\sqrt{\mb{\hat{v}}_{t-1}}}, -\mb{g}_t \right\rangle = \left\langle\frac{\mb{m}_{t-1}}{\sqrt{\mb{\hat{v}}_{t-1}}}, \mb{g}_t\right\rangle
	\end{equation}
\end{small}
Here, the operator $\left\langle\cdot, \cdot\right\rangle$ denotes the angle between two vectors. This expression is easy to understand, since the $-\mb{m}_{t-1} / \sqrt{\mb{\hat{v}}_{t-1}}$ can represent the updating direction of $(t-1)_{th}$ iteration in {\ams}, meanwhile $-\mb{g}_t$ is the reverse of the present gradient. So we can simplify it as $\theta = <\frac{\mb{m}_{t-1}}{\sqrt{\mb{\hat{v}}_{t-1}}}, \mb{g}_t>$. If $\theta$ is close to zero (denoted by $\theta\rightarrow 0^\circ$), the $\mb{m}_{t-1} / \sqrt{\mb{\hat{v}}_{t-1}}$ (previous update volume) and $\mb{g}_t$ are almost in the same direction, and the weights for them will not be very important. Meanwhile, if $\theta$ approaches $180^\circ$ (denoted by $\theta\rightarrow 180^\circ$), the previous update volume and $\mb{g}_t$ will be in totally reverse directions. This means in the current step, the previous momentum term is already in a wrong direction. Therefore, to rectify this error of the last momentum, {\our} proposes to assign the current gradient's weight (i.e., $\beta_{1,t}$ in our paper) with a larger value instead. As the $\beta_{1,t}$ varies when $\theta$ changes from $0^\circ$ to $180^\circ$, we intend to define $\beta_{1,t}$ with the following equation:
\begin{small}
	\begin{equation}
	\beta_{1,t} =  \begin{cases}
	\sin \ K + \epsilon &  \mbox{$\theta \in [0, \frac{\pi}{2})$} \\
	1 / K & \mbox{$\theta\in [\frac{\pi}{2}, \pi]$}
	\end{cases}
	\end{equation}\label{equ:beta_1}
\end{small}
where $K = 10(2+\pi)/2\pi$ and $\epsilon$ is a very small value (e.g., $\epsilon = 0.001$). 
In the equation above, the threshold of the piecewise function is $\theta = \pi/2$, because $\sin\theta$ comes to the maximum at this point and goes down when $\theta > \frac{\pi}{2}$. If $\frac{\pi}{2}\leq \theta\leq \pi$, which is exactly the situation $\theta\rightarrow 180^\circ$ we discussed above, we intend to keep $\beta_{1,t}$ in a relatively large value. The reason we rescale $\sin\theta$ by ${1}/{K}$ is that directly applying $\beta_{1,t}= \sin\theta$ will overweight $\mb{g}_t$, which may cause fluctuations on the update routes. The value of $K$ is determined by:
\begin{small}
	\begin{equation}
	K = 10(\int_{0}^{\frac{\pi}{2}} \sin\theta d\theta + \int_{\frac{\pi}{2}}^{\pi} 1 d\theta) = \frac{10(2+\pi)}{2\pi}
	\end{equation}\label{equ:K}
\end{small}
In the equation above, assume $\theta$ is randomly distributed on $[0, \pi]$, in this calculation we can get 
\begin{small}
	\begin{equation}
	\mathbb{E}[\beta_{1,t}] = \frac{1}{\pi}\int_{0}^{\pi} \beta_{1,t}(\theta) d\theta = 0.1
	\end{equation}\label{equ:E[beta_1]}
\end{small}
In other words, the expectation of $\beta_{1,t}$ (i.e., $\mathbb{E}(\beta_{1,t})$) will be identical to the $\beta_1$ used in {\adam}~\cite{ADAM}. 
After obtaining $\beta_{1,t}$, it will be applied to calculating $\mb{m}_t$ as shown in Algorithm~\ref{alg:cwm}. In this way, we have achieved momentum with adaptive weights, and this weight is automatically computed during the training process, fewer hyperparameters will be involved.

\subsection{Backtrack Term $d_t$}
To further speed up the convergence rate, we employ a novel backtrack mechanism for {\our}. As a mechanism computed based on the discriminative angle $\theta$, the backtrack term allows {\our} to eliminate redundant update in each iteration.
Besides, according to our following analysis, the backtrack term $d_t$ virtually collaborates with the $\beta_{1,t}$ term to further accelerate the convergence of the training process. 
\vspace{-5pt}
\subsubsection{Motivation}\label{sec:motivation}
When optimizer (e.g., {\adam}) updates variables of the loss function (e.g., $f(x, y)$), some update routes will look like the black arrow lines shown in Figure~\ref{fig:zigzag}, especially when the discriminative angle $\theta$ is larger than $90^\circ$. We call this phenomenon the "zig-zag" route. In Figure~\ref{fig:zigzag}, it shows the update routes of a 2-dimension function. Each black arrow line in the figure represents the variables' update in each epoch; the red dashed line is the direction of the update routes; the $\theta$ is the discriminative angle. If $\theta \geq 90^\circ$, the "zig-zag" phenomenon will appear severely, which may lead to slower convergence speed. The main reason is when $\theta \geq 90^\circ$, if we map two neighboring update directions onto the coordinate axes, there will be at least one axis of the directions being opposite. This situation is shown in Figure~\ref{fig:axis_decomp}. For the example of a function with 2-dimension variables, the update volume $\mb{m}_1/\sqrt{\mb{v}_1}$ can be decomposed into $(x_1, y_1)^\top$ in Figure~\ref{fig:axis_decomp}, and the same with $\mb{m}_2/\sqrt{\mb{v}_2}$. We can notice that $y_1$ and $y_2$ are in the opposite directions, so the first and second steps practically have inverse updates subject to the $y$ axis. We attribute this situation to the over update (or redundant update) of the first step. Therefore the backtrack term $d_t$ is proposed to restrict this situation.

\subsubsection{Mechanism}\label{subsubsec:dt_m}
Since the redundant update situation is caused by over updating of the previous iteration, simply we intend to deal with it through a backward step. Meanwhile, during the updating process of variables, not every step will suffer from the redundant update: if $\theta\rightarrow 0^\circ$, the updating process becomes smooth, not like the situation shown in Figure~\ref{fig:zigzag}. Besides, from the analysis above we conclude that if $\theta \geq 90^\circ$, there will be at least one dimension involves the redundant update. Thus, in the $t_{th}$ iteration we quantify $d_t$ as the following equation:
\begin{small}
	\begin{equation}
	d_t = \min\{0.5\cos\theta, 0\}
	\end{equation}\label{equ:d_t}
\end{small}
and we rewrite the updating term with backtrack in {\our} as
\begin{small}
	\begin{equation}
	\mb{\Delta}_t = d_t\cdot\mb{\Delta}_{t-1} - \eta_t\cdot\frac{\mb{m}_t}{\sqrt{\mb{\hat{v}}_t}}
	\end{equation}\label{equ:Delta_t}
\end{small}
where $\theta$ is the discriminative angle and $\mb{\Delta}_t$ is the updating term in Algorithm~\ref{alg:cwm}. By designing $d_t$ in this way, when $\theta\rightarrow 0^\circ$, $d_t = 0$ and there is no backward step, the updating term $\mb{\Delta}_t = -\eta_t\cdot\frac{\mb{m}_t}{\sqrt{\mb{\hat{v}}_t}}$ is similar to {\ams}; when $\theta\rightarrow 180^\circ$, $d_t = 0.5\cos\theta$ and comes to the maximum value when $\theta = 180^\circ$. The reason that $\cos\theta$ is rescaled by $0.5$ is that: in Figure~\ref{fig:theta180}, $\mb{w}_{t-1}$ and $\mb{w}_t$ are the variables updated by {\our} without $d_t$ term in the $(t-1)_{th}$ and $t_{th}$ iterations respectively. If the backtrack mechanism is implemented, in the $(t+1)_{th}$ iteration, since $\theta = 180^\circ$, firstly $d_t = 0.5\cos\theta \to -0.5$ makes the backtrack to the $\mb{w'}_t$ point (the middle point of $\mb{w}_{t-1}$ and $\mb{w}_t$). Thus, this backtrack step allows the variable to further approach the optima.

By implementing the backtrack term $d_t$, {\our} can combine it with the adaptive momentum weight $\beta_{1,t}$ to achieve the collaborating of them. For the situation of large discriminative angle ($\theta > 90^\circ$), both $\beta_{1,t}$ and $d_t$ in the current step can make corrections to the last update. Since when $\theta > 90^\circ$, the last update is in conflict direction compared with the current gradient, and  $\beta_{1,t}$ will increase to allocate a large weight for the present gradient, which subsequently corrects the previous step. Meanwhile, the $d_t$ will also conduct a backward step of to further rectify the last update.

\begin{figure*}[t]
	\centering
	\subfigure[Train loss on ORL]{
		\begin{minipage}[l]{0.48\columnwidth}
			\centering
			\includegraphics[width=0.85\textwidth]{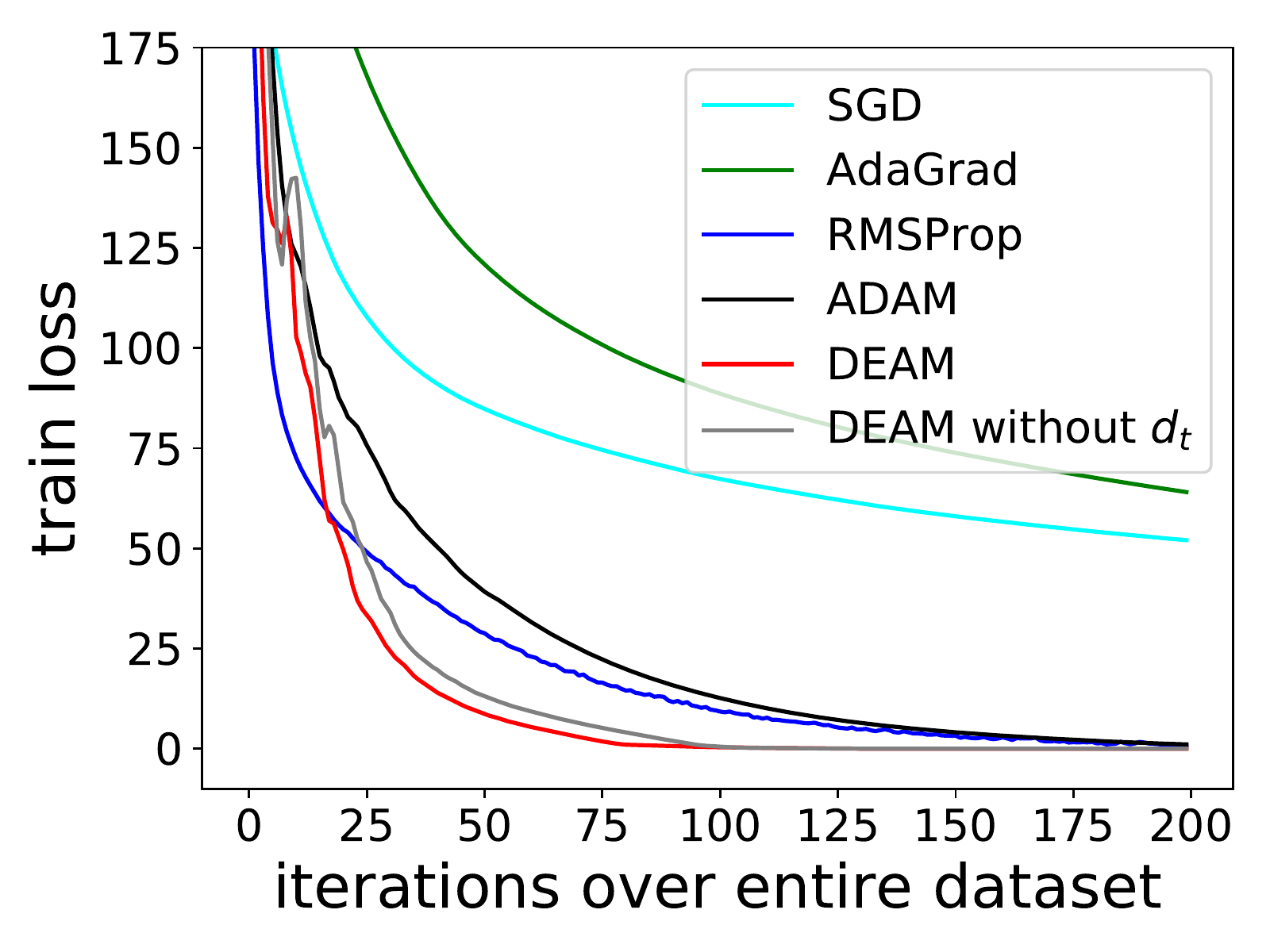}
		\end{minipage}
		\label{fig:orl_1}
	}
	\subfigure[Test loss on ORL]{
		\begin{minipage}[l]{0.48\columnwidth}
			\centering
			\includegraphics[width=0.85\textwidth]{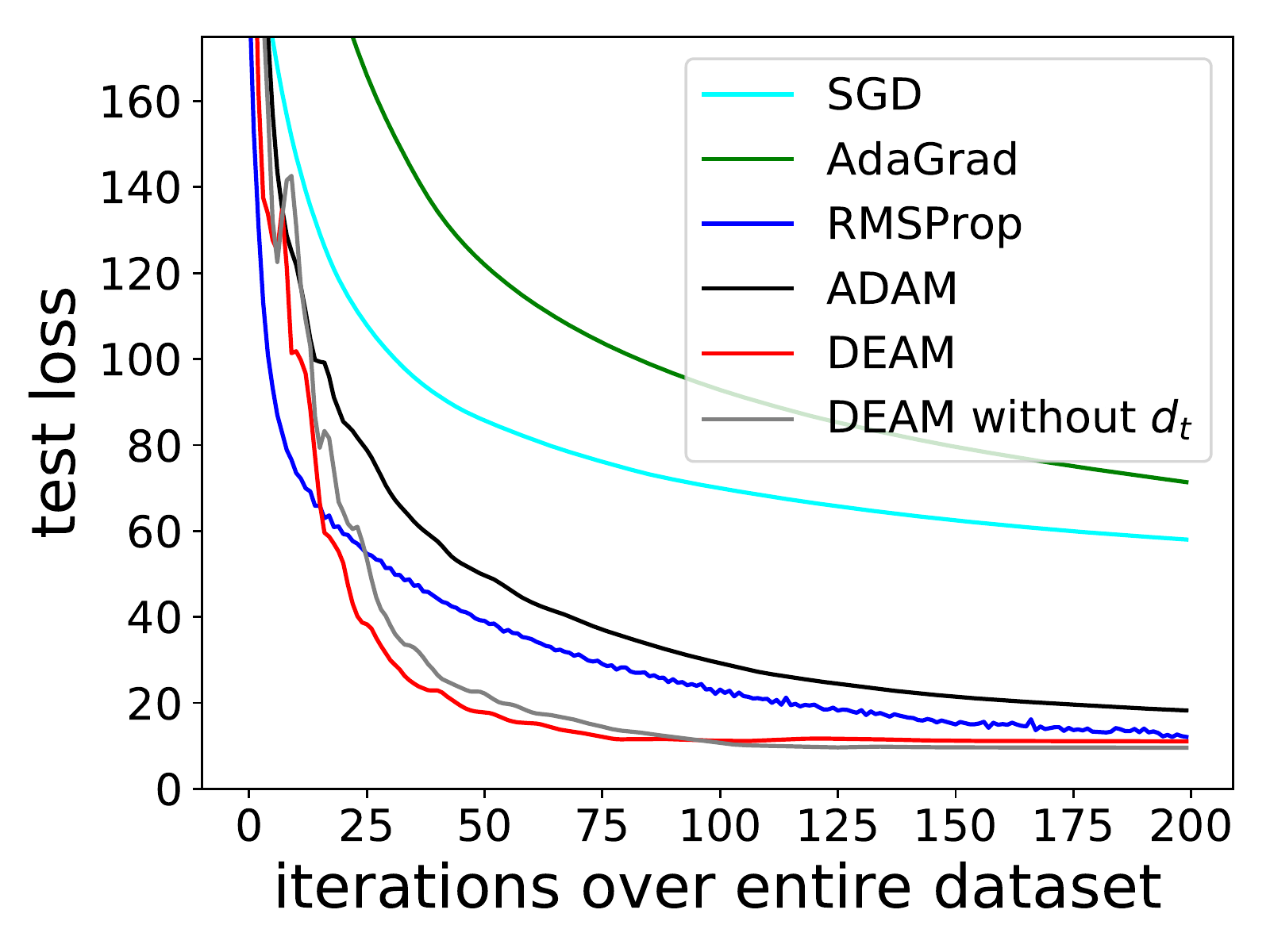}
		\end{minipage}
		\label{fig:orl_2}
	}
	\subfigure[Train loss on MNIST]{ \label{fig:mnist_train}
		\begin{minipage}[l]{0.48\columnwidth}
			\centering
			\includegraphics[width=0.85\textwidth]{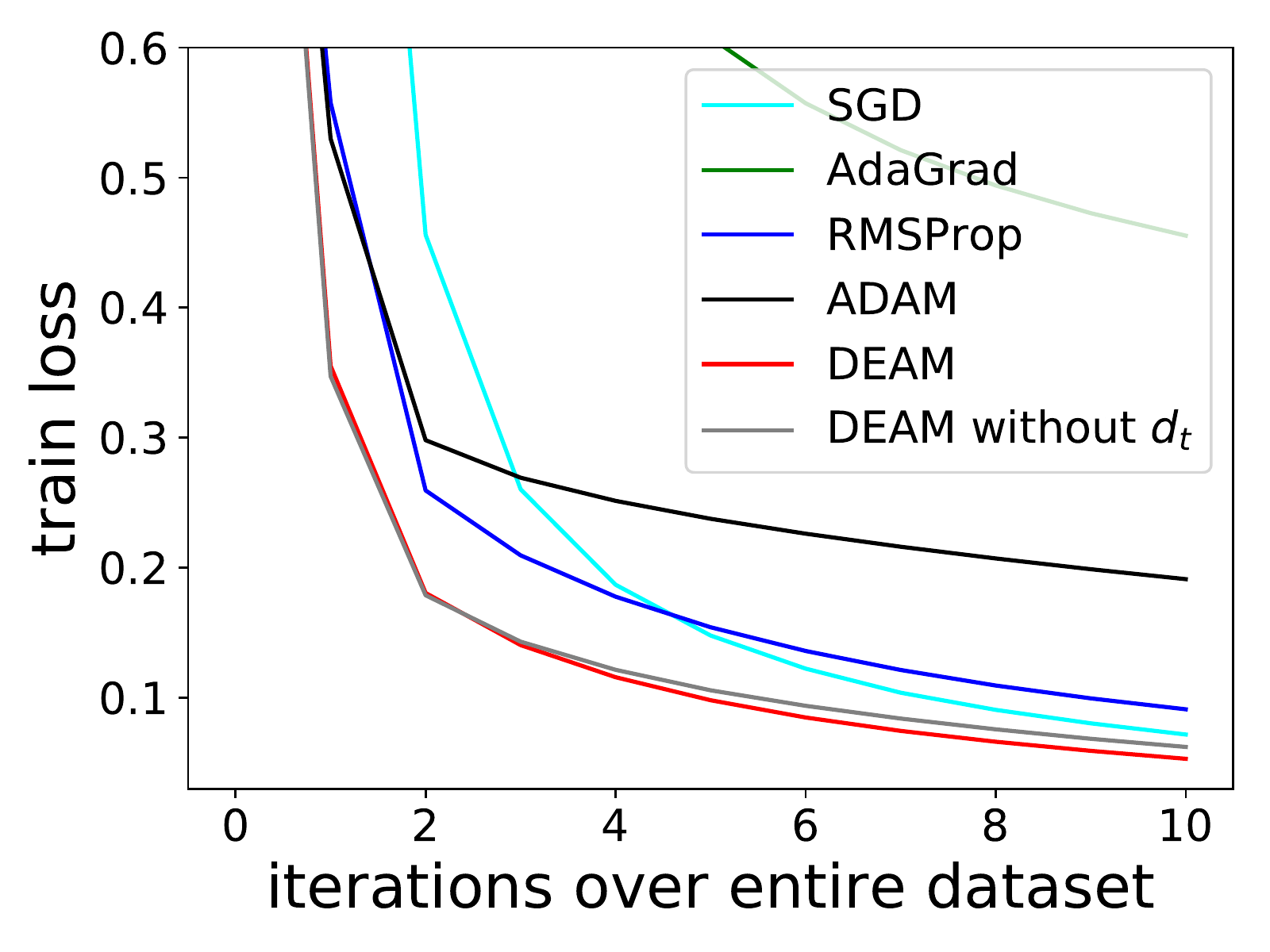}
		\end{minipage}
		\label{fig:dnn_1}
	}
	\subfigure[Test loss on MNIST]{ \label{fig:mnist_test}
		\begin{minipage}[l]{0.48\columnwidth}
			\centering
			\includegraphics[width=0.85\textwidth]{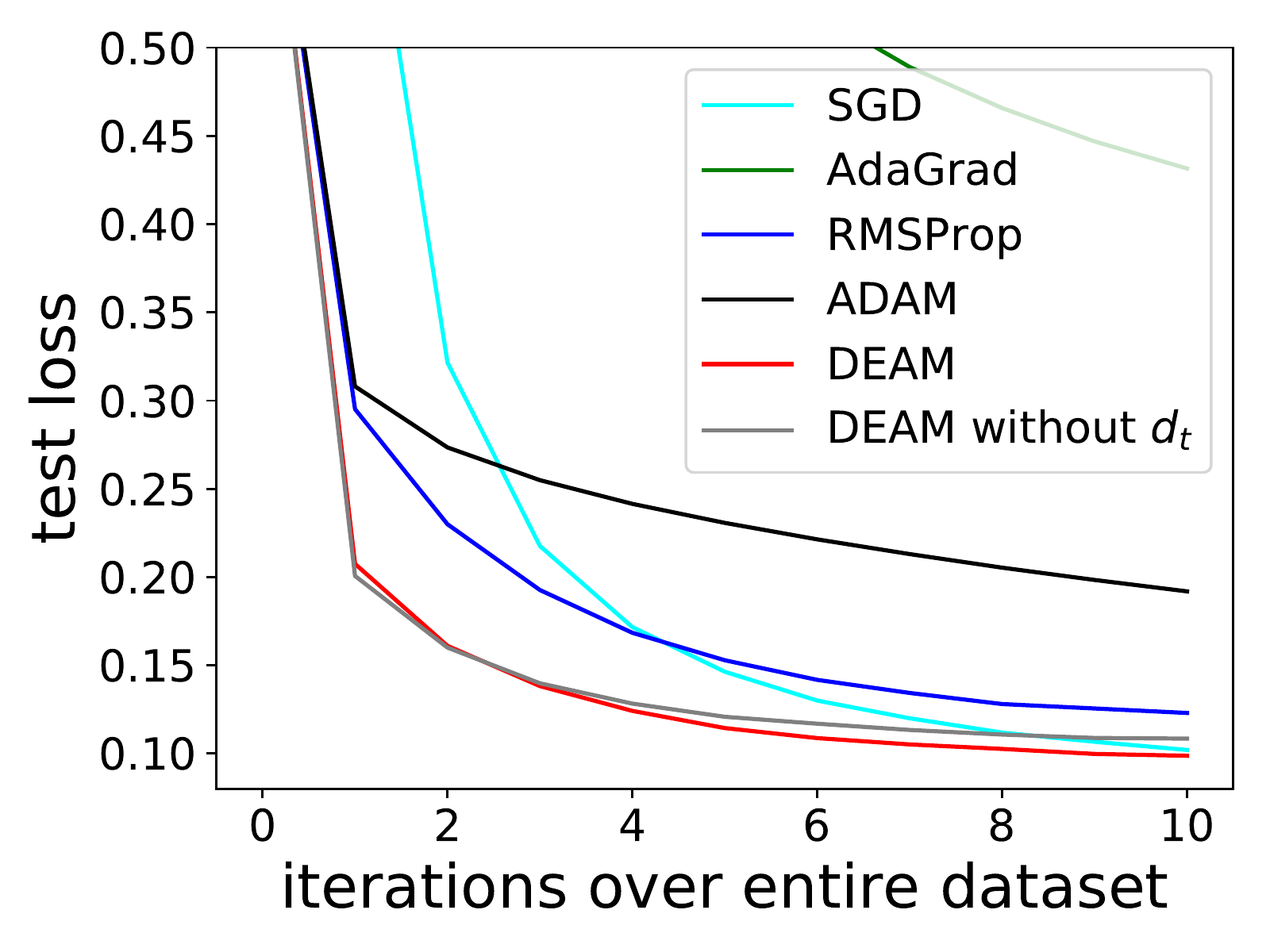}
		\end{minipage}
		\label{fig:dnn_2}
	}
	\caption{Results of Logistic Regression and DNN structures}\label{fig:dnn}
\end{figure*}
\subsection{Theoretical Analysis}\label{sec: convergence_analysis}
In this part, we give the detailed analysis on the convergence of our  {\our} algorithm. According to~\cite{ADAM,amsgrad,online_convex_program,adafom}, given an arbitrary sequence of convex objective functions $f_1(\mb{w}), f_2(\mb{w}),\dots, f_T(\mb{w})$, we intend to evaluate our algorithm using the regret function, which is denoted as:
\begin{small}
	\begin{equation}
	R(T) = \sum_{t=1}^{T}[f_t(\mb{w}_t) - f_t(\mb{w}^*)]
	\end{equation}\label{equ:Delta_t}
\end{small}
where $\mb{w}^*$ is the globally optimal point. In the following Theorem~\ref{thm:converge}, we will show that the above regret function is bounded. Before proving the Theorem~\ref{thm:converge}, there are some definitions and lemmas as the pre-requisites.
\begin{defn} \label{def:convex}
	If a function $f: \mathbb{R}^d \rightarrow \mathbb{R}$ is convex, then $\forall x, y\in\mathbb{R}^d$, $\forall \phi\in[0,1]$, we have
	\[
	f(\phi x + (1-\phi)y)\leq\phi f(x) + (1-\phi)f(y)
	\]
\end{defn}
\begin{defn}\label{lemma:bound}
	If a function $f: \mathbb{R}^d\rightarrow \mathbb{R}$ is convex, then $\forall x, y\in\mathbb{R}^d$ we have
	\[
	f(y)\geq f(x) + \nabla f(x)^\top (y-x)
	\]
\end{defn}

\begin{lemma}\label{lemma:m_t}
	Assume that the function $f_t$ has bounded gradients, $\left\|\nabla f_t(\mb{w})\right\|_\infty\leq G_\infty$. Let  $\mb{m}_{t, i}$ represents the $i_{th}$ element of $\mb{m}_t$ in {\our}, then the $\mb{m}_{t, i}$ is bounded by
	\[
	\mb{m}_{t,i}\leq \frac{(1-\epsilon_0)G_\infty}{K(1-\lambda)}
	\]
\end{lemma}
\begin{proof}
	Let $g_t = \nabla f_t(\mb{w})$.
	According to the definition of $\mb{m}_{t,i}$ in our algorithm,
	{\scriptsize
		\begin{align*}
		\mb{\mb{m}}_{t,i} &= \sum_{j=1}^{t}\beta_{1,j}\prod_{l=1}^{t-j}(1-\beta_{1,t-l+1})g_{j,i}\\
		&\leq \frac{G_\infty}{K}\sum_{j=1}^{t}\prod_{l=1}^{t-j}(1-\epsilon)
		\leq \frac{G_\infty}{K}\sum_{j=1}^{t}(1-\epsilon_0)\lambda^{t-j}\leq \frac{(1-\epsilon_0)G_\infty}{K(1-\lambda)}
		\end{align*}
	}
	where $K$ and $\epsilon$ are the terms in Algorithm~\ref{alg:cwm}.
\end{proof}
For the following proof, $\mb{g}_t := \nabla f_t(\mb{w}_t)$ and $\mb{g}_{t, i}$ will represent the $i_{th}$ element of $\mb{g}_t\in\mathbb{R}^d$, and $\mb{g}_{1:t, i} = [\mb{g}_{1,i},\mb{g}_{2,i},\dots ,\mb{g}_{t,i}]$.
\begin{thm}\label{thm:converge}
	Assume $\{f_t\}_{t=1}^T$ have bounded gradients $\left\|\nabla f_t(\mb{w})\right\|_\infty\leq G_\infty$ for all $\mb{w}\in \mathbb{R}^d$, all variables are bounded by $\left\|\mb{w}_p - \mb{w}_q\right\|_2\leq D$ and $\left\|\mb{w}_p - \mb{w}_q\right\|_\infty\leq D_\infty$, $\forall p,q \in \{1,2,\dots,T\}$, $\eta_t = \eta/\sqrt{t}$, $\gamma_1 = (1-\epsilon_0) / \sqrt{\beta_2}$ and satisfies $\gamma_1<1$, $\epsilon = 1-(1-\epsilon_0)\lambda^{t-1}, \lambda\in(0,1)$. Our proposed algorithm can achieve the following bound on regret:
	{\scriptsize
		\begin{align*}
		R(t)&\leq \frac{D^2}{\epsilon_0\eta}\sum_{i=1}^{d}\sqrt{T\mb{\hat{v}}_{T,i}}
		+ \frac{(1-\epsilon_0)^2G_\infty D_\infty d}{K(1-\lambda)^2\epsilon_0}\\
		&+ \frac{\eta \sqrt{1+\log T}}{2\epsilon^2_0(1-\gamma_1)\sqrt{1-\beta_2}}\sum_{i=1}^{d}\left\|\mb{g}_{1:T,i}\right\|_2
		\end{align*}
	}
\end{thm}
\begin{proof}
	According to Definition~\ref{lemma:bound}, for $\forall t \in\{1,2,\dots,T\}$, we have
	{\scriptsize
		\begin{align*}
		f_t(\mb{w}_t) - f_t(\mb{w}^*) \leq\nabla f_t(\mb{w}_t)^\top (\mb{w}_t - \mb{w}^*)
		= \sum_{i=1}^{d} \mb{g}_{t, i}(\mb{w}_{t,i} - \mb{w}_{i}^*)
		\end{align*}
	}
	From the definition of $\mb{\Delta}_{t}$ in the updating rule of {\our}, we know it is equal to multiplying the learning rate $\eta_t$ in some iterations by a number in $[0.5, 1]$, which means $\mb{w}_{t+1} = \mb{w}_t - \hat{\eta}_t\cdot\frac{\mb{m}_t}{\sqrt{\mb{\hat{v}}_t}};\hat{\eta}_t = \mu_t\cdot\eta_t$, where $\mu_t\in [0.5, 1]$. 
	If we first focus on the $i_{th}$ element of $\mb{w}_t$, we can get
	{\scriptsize
		\begin{align*}
		&(\mb{w}_{t+1,i} - \mb{w}^*_i)^2= (\mb{w}_{t,i} - \mb{w}^*_i - \hat{\eta}_t\cdot\frac{\mb{m}_t}{\sqrt{\mb{\hat{v}}_t}})^2 =(\mb{w}_{t,i} - \mb{w}^*_i)^2 \\
		-& 2\hat{\eta}_t(\frac{(1 - \beta_{1,t})}{\sqrt{\mb{\hat{v}}_{t,i}}}\mb{m}_{t-1,i} + \frac{\beta_{1,t}}{\sqrt{\mb{\hat{v}}_{t,i}}}\mb{g}_{t,i})\cdot(\mb{w}_{t,i} - \mb{w}^*_i)+ (\hat{\eta}_t\cdot\frac{\mb{m}_t}{\sqrt{\mb{\hat{v}}_t}})^2
		\end{align*}
	}
	Then,
	{\scriptsize
		\begin{align*}
		&2\hat{\eta}_t\cdot\frac{\beta_{1,t}}{\sqrt{\mb{\hat{v}}_{t,i}}}\mb{g}_{t,i} (\mb{w}_{t,i} - \mb{w}^*_i) = (\mb{w}_{t,i} -\mb{w}^*_i)^2 - (\mb{w}_{t+1,i} -\mb{w}_i^*)^2\\
		&- 2\hat{\eta}_t\cdot\frac{(1 - \beta_{1,t})}{\sqrt{\mb{\hat{v}}_{t,i}}}\mb{m}_{t-1,i}\cdot(\mb{w}_{t,i} - \mb{w}^*_i) + \hat{\eta}^2_t\cdot\frac{\mb{m}^2_{t,i}}{\mb{\hat{v}}_{t,i}}
		\end{align*}
	}
	So we can obtain
	{\scriptsize
		\begin{align}
		\mb{g}_{t,i} (\mb{w}_{t,i} - \mb{w}^*_i) &= \frac{\sqrt{\mb{\hat{v}}_{t,i}}}{2\hat{\eta}_t\beta_{1,t}} [(\mb{w}_{t,i} - \mb{w}^*_i)^2 - (\mb{w}_{t+1,i} - \mb{w}^*_i)^2]\\
		&~~~~ + \frac{(1-\beta_{1,t})}{\beta_{1,t}}\mb{m}_{t-1,i}(\mb{w}^*_i - \mb{w}_{t,i})\\
		&~~~~ +\frac{\hat{\eta}_t}{2\beta_{1,t}} \cdot\frac{\mb{m}^2_{t,i}}{\sqrt{\mb{\hat{v}}_{t,i}}}
		\end{align}\label{fmla:1}
	}
	For the right part of $(9)$ in the above formula, if we sum it from $t=1$ to $t = T$,
	{\scriptsize
		\begin{align*}
		&\sum_{t=1}^{T}\frac{\sqrt{\mb{\hat{v}}_{t,i}}}{2\hat{\eta}_t\beta_{1,t}}[(\mb{w}_{t,i} - \mb{w}^*_i)^2 - (\mb{w}_{t+1,i} - \mb{w}^*_i)^2]\\ 
		\leq& \frac{1}{\epsilon_0}\{(\mb{w}_{1,i} - \mb{w}^*_i)^2\cdot\frac{\sqrt{\mb{\hat{v}}_{1,i}}}{\eta_1}
		+\dots + (\mb{w}_{T,i} - \mb{w}^*_i)^2(\frac{\sqrt{\mb{\hat{v}}_{T,i}}}{\eta_T} - \frac{\sqrt{\mb{\hat{v}}_{T-1,i}}}{\eta_{T-1}})\}\\ 
		\leq& \frac{D^2}{\epsilon_0\eta}\sqrt{T\mb{\hat{v}}_{T,i}}
		\end{align*}
	}
	The first inequality is satisfied because of the line 13 in Algorithm~\ref{alg:cwm}. For the $(10)$ in the formula, if we sum it from $t=1$ to $t = T$,
	{\scriptsize
		\begin{align*}
		&\sum_{t=1}^{T}\frac{(1-\beta_{1,t})}{\beta_{1,t}}\mb{m}_{t-1,i}(\mb{w}^*_i - \mb{w}_{t,i})
		\leq\frac{(1-\epsilon_0)G_\infty D_\infty}{K(1-\lambda)\epsilon_0} \sum_{t=1}^{T}(1-\beta_{1,t})\\
		\leq& \frac{(1-\epsilon_0)G_\infty D_\infty}{K(1-\lambda)\epsilon_0}\sum_{t=1}^{T}(1 - \epsilon)
		= \frac{(1-\epsilon_0)G_\infty D_\infty}{K(1-\lambda)\epsilon_0}\sum_{t=1}^{T}(1 - \epsilon_0)\lambda^{t-1}\\
		\leq& \frac{(1-\epsilon_0)^2G_\infty D_\infty}{K(1-\lambda)^2\epsilon_0}
		\end{align*}
	}
	\begin{figure*}[t]
		\centering
		\subfigure[Train loss on ORL]{\label{fig:orl_train}
			\begin{minipage}[l]{0.65\columnwidth}
				\centering
				\includegraphics[width=0.85\textwidth]{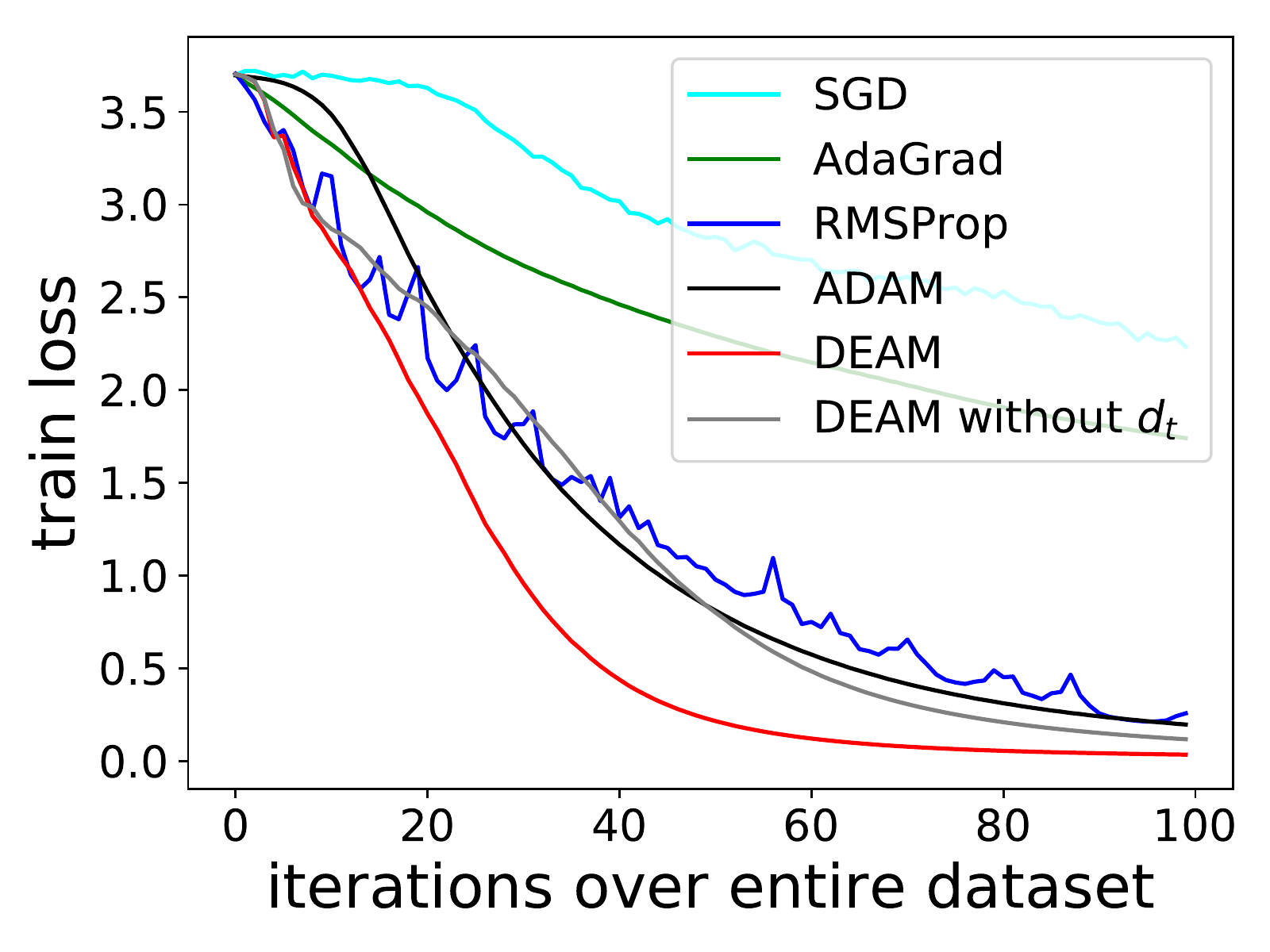}
			\end{minipage}
		}
		\subfigure[Test loss on ORL]{ \label{fig:orl_test}
			\begin{minipage}[l]{0.65\columnwidth}
				\centering
				\includegraphics[width=0.85\textwidth]{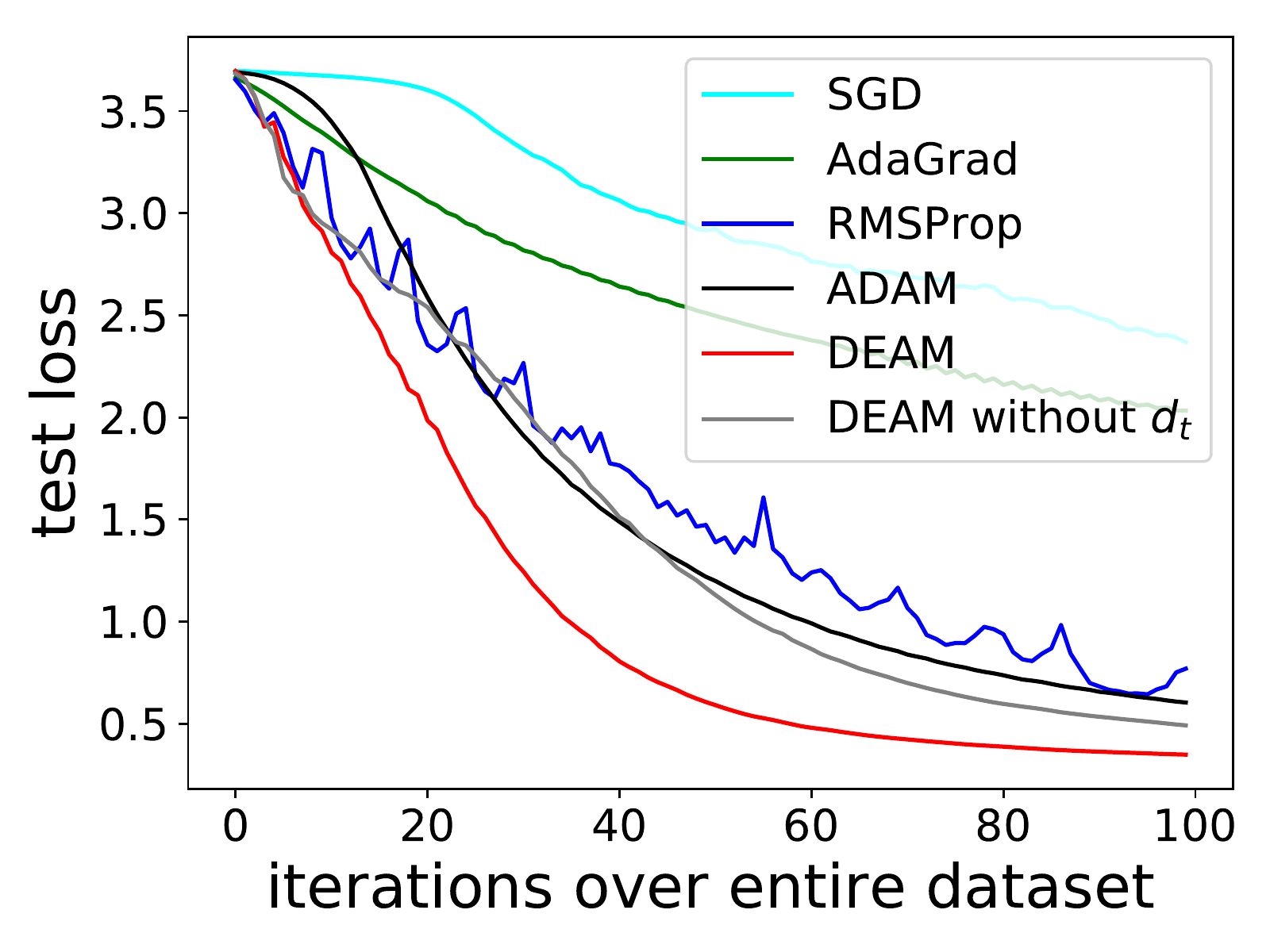}
			\end{minipage}
		}
		\vspace{-5pt}
		\subfigure[Train loss on MNIST]{ \label{fig:mnist_cnn_train}
			\begin{minipage}[l]{0.65\columnwidth}
				\centering
				\includegraphics[width=0.85\textwidth]{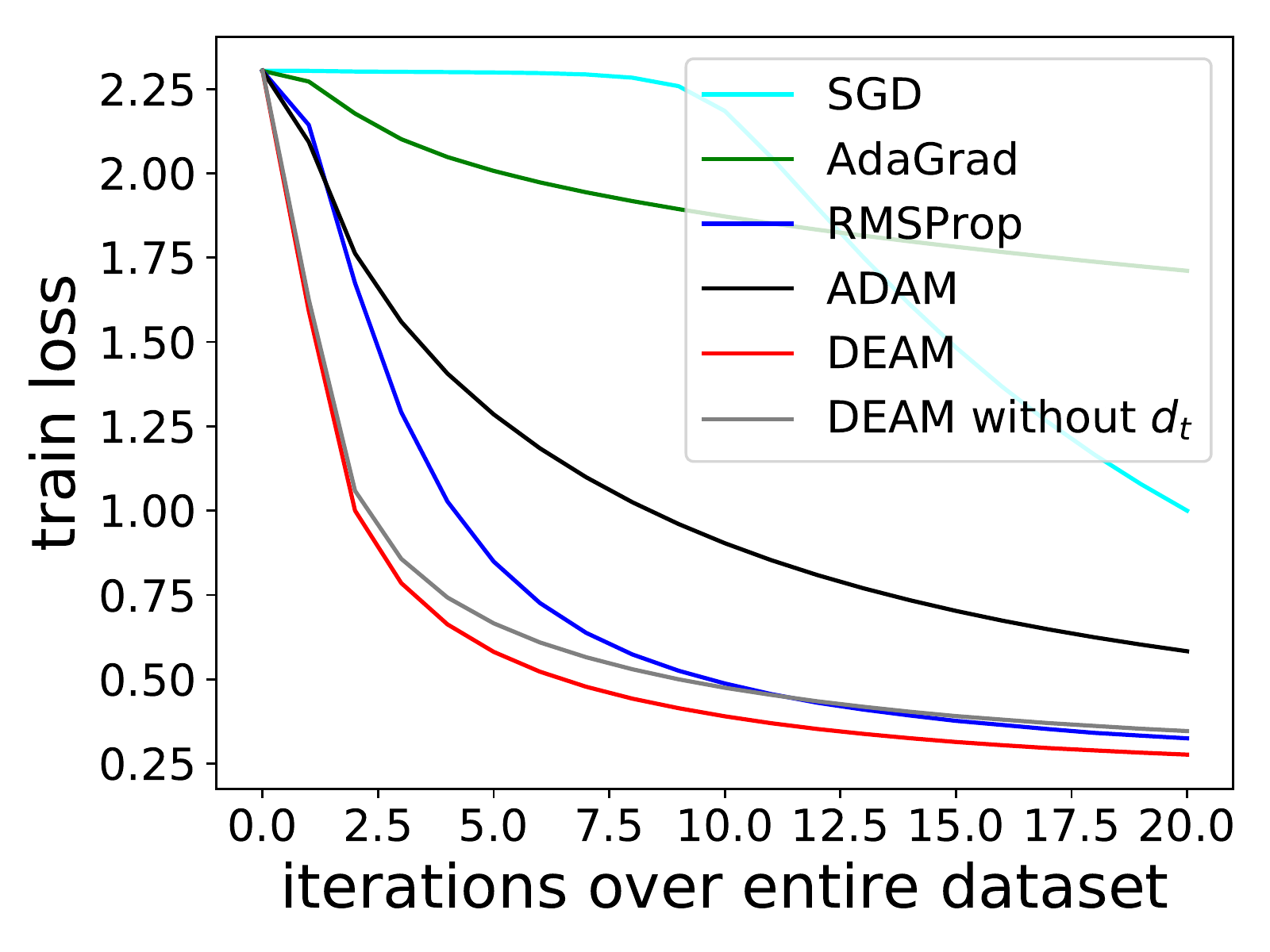}
			\end{minipage}
		}
		\subfigure[Test loss on MNIST]{ \label{fig:mnist_cnn_test}
			\begin{minipage}[l]{0.65\columnwidth}
				\centering
				\includegraphics[width=0.85\textwidth]{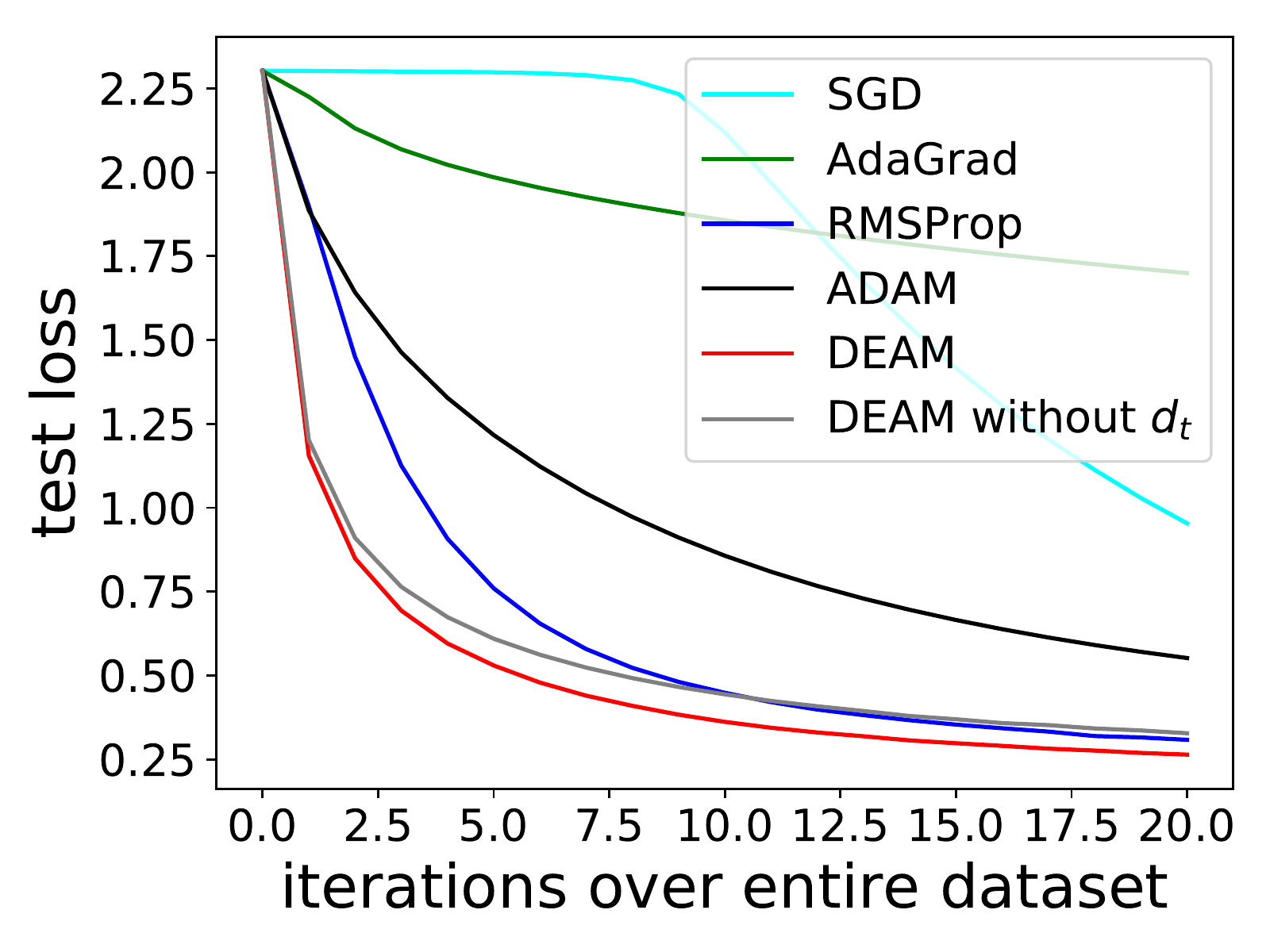}
			\end{minipage}
		}
		\subfigure[Train loss on CIFAR]{ \label{fig:cifar_train}
			\begin{minipage}[l]{0.65\columnwidth}
				\centering
				\includegraphics[width=0.85\textwidth]{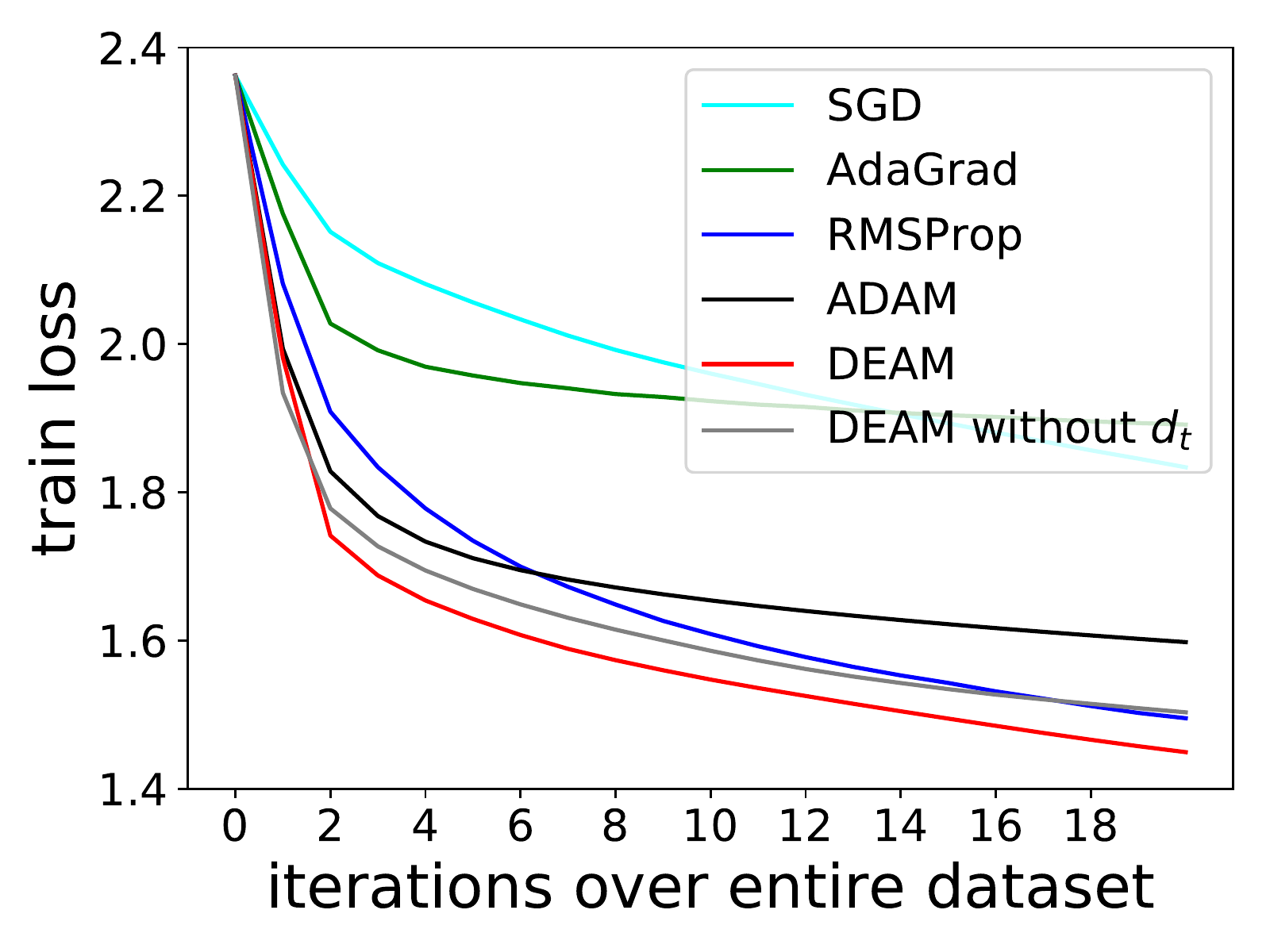}
			\end{minipage}
		}
		\subfigure[Test loss on CIFAR]{ \label{fig:cifar_test}
			\begin{minipage}[l]{0.65\columnwidth}
				\centering
				\includegraphics[width=0.85\textwidth]{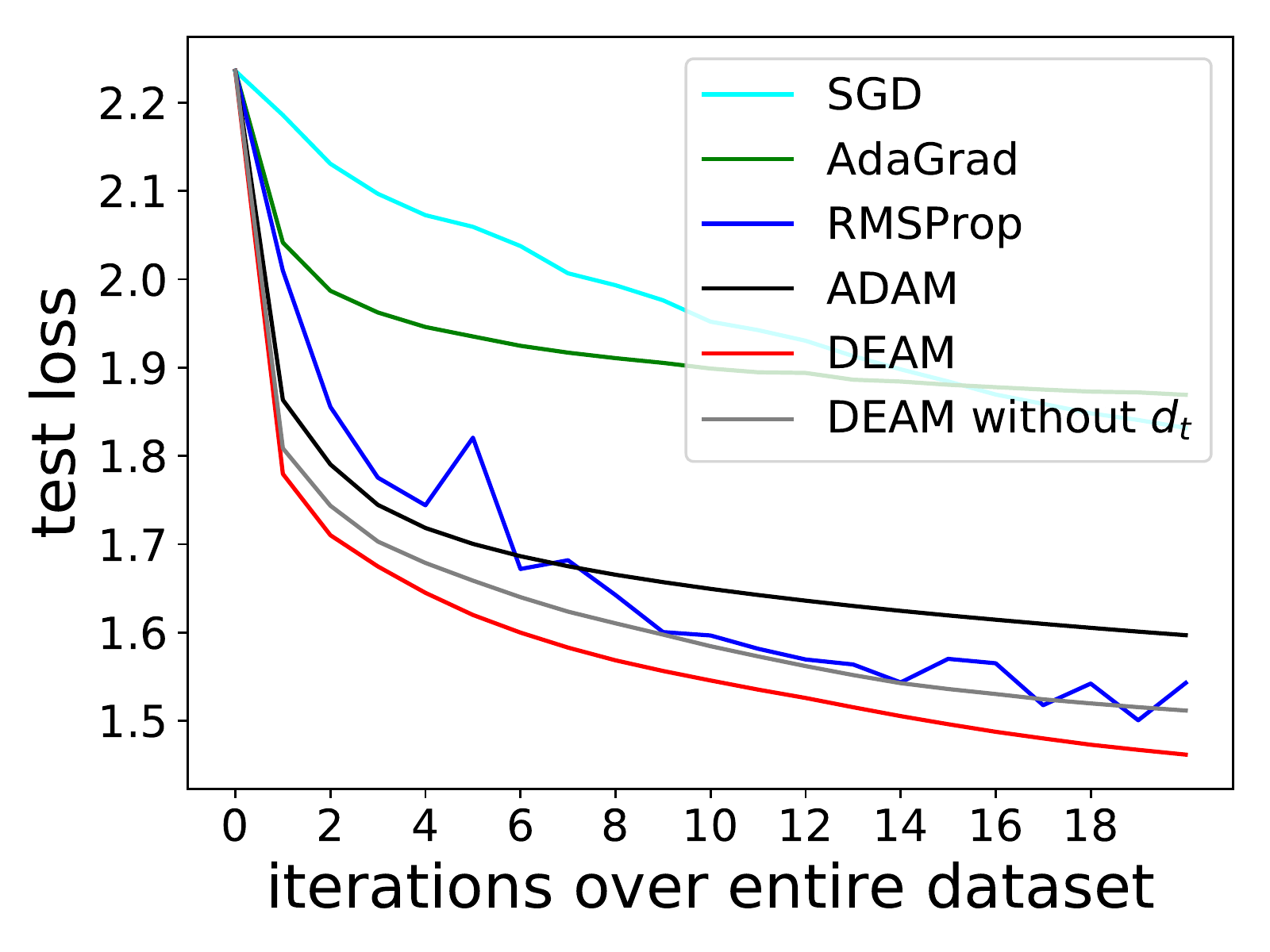}
			\end{minipage}
		}
		\vspace{-10pt}
		\caption{Results of CNN structure}\label{fig:cnn}
	\end{figure*}
	\begin{table*}[t]
		\renewcommand\arraystretch{1.1}
		\scriptsize
		\centering
		\begin{threeparttable}
			\caption{Running time of {\our} and comparison methods (the unit of values is second)}
			\begin{tabular}{|c|c|c|c|c|c|}
				\hline 
				\multirow{2}*{\tabincell{c}{\textbf{Comparison}\\\textbf{Methods}}}&\multicolumn{5}{|c|}{Running time on all models}\\
				
				\cline{2-6}
				&\tabincell{c}{Logistic Regression \\ on ORL}&DNN on MNIST&CNN on ORL&CNN on MNIST&CNN on CIFAR-10\\
				\cline{1-6}
				{\our}&\textbf{38}&\textbf{302}&\textbf{35064}&\textbf{11679}&\textbf{57761}\\
				\hline
				{\adam}&102&664&47418&21775&67584\\ 
				\hline
				{\rmsprop}&48&307&36722&11997&84305\\
				\hline
				{\adagrad}&$>200$&667&$>100000$&$>50000$&$>100000$\\
				\hline
				SGD&$>200$&346&$>100000$&16985&67564\\
				\hline
			\end{tabular}
			\label{tab:running_time}
		\end{threeparttable}
	\end{table*}
	The first inequality is according to Lemma~\ref{lemma:m_t}.	Finally, we will infer the $(11)$ in previous formula. According to the Lemma 2 of~\cite{amsgrad}, we have
	{\scriptsize
		\begin{align*}
		\small
		&\sum_{t=1}^{T}\frac{\hat{\eta}_t}{2\beta_{1,t}}\cdot\frac{\mb{m}^2_{t,i}}{\sqrt{\mb{\hat{v}}_{t,i}}}\leq \frac{1}{2\epsilon_0}\sum_{t=1}^{T}\eta_t\frac{\mb{m}^2_{t,i}}{\sqrt{\mb{v}_{t,i}}}
		\leq \frac{\eta}{2\epsilon_0}\sum_{t=1}^{T}\frac{1}{\sqrt{t}}\cdot\frac{\mb{m}^2_{t,i}}{\sqrt{\mb{v}_{t,i}}}\\
		\leq& \frac{\eta}{2\epsilon_0} \sum_{t=1}^{T}\frac{(\sum_{j=1}^{t}(1-\epsilon_0)^{t-j})(\sum_{j=1}^{t}(1-\epsilon_0)^{t-j}\mb{g}^2_{j,i})}{\sqrt{t((1-\beta_2)\sum_{j=1}^{t}\beta_2^{t-j}\mb{g}^2_{j,i})}}\\
		\leq& \frac{\eta }{2\epsilon^2_0\sqrt{1-\beta_2}}\sum_{t=1}^{T}\frac{1}{\sqrt{t}}\sum_{j=1}^{t}\frac{(1-\epsilon_0)^{t-j}\mb{g}^2_{j,i}}{\sqrt{\beta_2^{t-j}\mb{g}^2_{j,i}}}\\
		\leq& \frac{\eta }{2\epsilon^2_0\sqrt{1-\beta_2}}\sum_{t=1}^{T}|\mb{g}_{t,i}|\sum_{j=t}^{T}\frac{\gamma_1^{j-t}}{\sqrt{t}} 
		\leq \frac{\eta \sqrt{1+\log T}}{2\epsilon^2_0(1-\gamma_1)\sqrt{1-\beta_2}}\left\|\mb{g}_{1:T,i}\right\|_2
		\end{align*}
	}
	In the above inequalities, some inferences are based on Cauchy-Schwarz Inequality. Therefore, the final bound of $R(T)$ can be expressed as
	{\scriptsize
		\begin{align*}
		R(T) &\leq \frac{D^2}{\epsilon_0\eta}\sum_{i=1}^{d}\sqrt{T\mb{\hat{v}}_{T,i}} + \frac{(1-\epsilon_0)^2G_\infty D_\infty d}{K(1-\lambda)^2\epsilon_0}\\ &+ \frac{\eta \sqrt{1+\log T}}{2\epsilon^2_0(1-\gamma_1)\sqrt{1-\beta_2}}\sum_{i=1}^{d}\left\|\mb{g}_{1:T,i}\right\|_2
		\end{align*}
	}
\end{proof}	
For the bound term, as $T\rightarrow +\infty$, $\frac{R(T)}{T}\rightarrow 0$ and we can infer that $\displaystyle\lim_{T\rightarrow \infty} [f_t(\mb{w}_t) - f_t(\mb{w}^*)] = 0$, which means the proposed algorithm can finally converge. 


\section{Experiments}\label{sec:experiment}
We have applied the {\our} algorithm on multiple popular machine learning and deep learning structures, both convex and non-convex situations. To show the advantages of the algorithm, we compare it with various popular optimization algorithms, including {\adam}~\cite{ADAM}, RMSProp~\cite{rmsprop}, AdaGrad~\cite{adagrad} and SGD. For all the experiments, the loss function (objective function) we have selected is the cross-entropy loss function, and the size of the minibatch is 128. Besides, the learning rate is 0.0001.

\subsection{Experiment Settings and Results}

\noindent\textbf{Logistic Regression}: We firstly evaluate our algorithm on the multi-class logistic regression model, since it is widely used and owns a convex objective function. We conduct logistic regression on the ORL dataset~\cite{ORL}. ORL dataset consists of face images of 40 people, each person has ten images and each image is in the size of $112\times 92$. The loss of objective functions on both training set and testing set are shown in Figure~\ref{fig:orl_1}, \ref{fig:orl_2}.

\noindent\textbf{Deep Neural Network}: We use deep neural network (DNN) with two fully connected layers of 1,000 hidden units and the Relu~\cite{relu} activation function. The dataset we use is MNIST~\cite{GBDR}. The MNIST dataset includes 60,000 training samples and 10,000 testing samples, where each sample is a $28\times 28$ image of hand-written numbers from 0 to 9. Result are exhibited in Figure~\ref{fig:dnn_1}, \ref{fig:dnn_2}.

\noindent\textbf{Convolutional Neural Network}: The CNN models in our experiments are based on the LeNet-5~\cite{GBDR}, and it is implemented on multiple datasets: ORL, MNIST and CIFAR-10~\cite{CIFAR}. The CIFAR-10 dataset consists of 60,000 $32\times 32$ images in 10 classes, with 6,000 images per class. For different datatsets, the structures of CNN models are modified: for the ORL dataset, the CNN model has two convolutional layers with 16 and 36 feature maps of $5\time 5$ kernels and 2 max-pooling layers, and a fully connected layer with 1024 neurons; for the MNIST dataset, the CNN structure follows the LeNet-5 structure in~\cite{GBDR}; for CIFAR-10 dataset, the CNN model consists of three convolutional layers with 64, 128, 256 kernels respectively, and a fully connected layer having 1024 neurons. All experiments apply Relu~\cite{relu} activation function. The results are shown in Figure~\ref{fig:cnn}.

We can observe that {\our} converges faster than other widely used optimization algorithms in all the cases. Within the same number of epoches, {\our} can converge to the lowest loss on both the training set and test set.




\subsection{Analysis of the Backtrack Mechanism}

To show the effectiveness of the backtrack term $d_t$, we also carry out the experiments of {\our} without $d_t$ term, and exhibit the results in Figures~\ref{fig:dnn}, \ref{fig:cnn}. The results indicate that after applying $d_t$ term, the converging speed will become slightly faster. To thoroughly prove the effectiveness of our proposed $d_t$ term, we also compare it with other definitions of backtrack terms e.g., $d_t = 0.5\cos\theta$, and exhibit the results in Figure~\ref{fig:d_t_other}. In Figure~\ref{fig:d_t_other}, $d_t = sigmoid \_ based$ represents that $d_t = -1/(1+ e^{-(\theta - \frac{1}{2}\pi)}) + \frac{1}{2}$, and $d_t = tanh \_ based$ means that $d_t = -2/(1 + e^{-2(x - \frac{1}{2}\pi)}) + 1$. 
\begin{figure}[t]
	\centering
	\includegraphics[width=5cm,height=3.5cm]{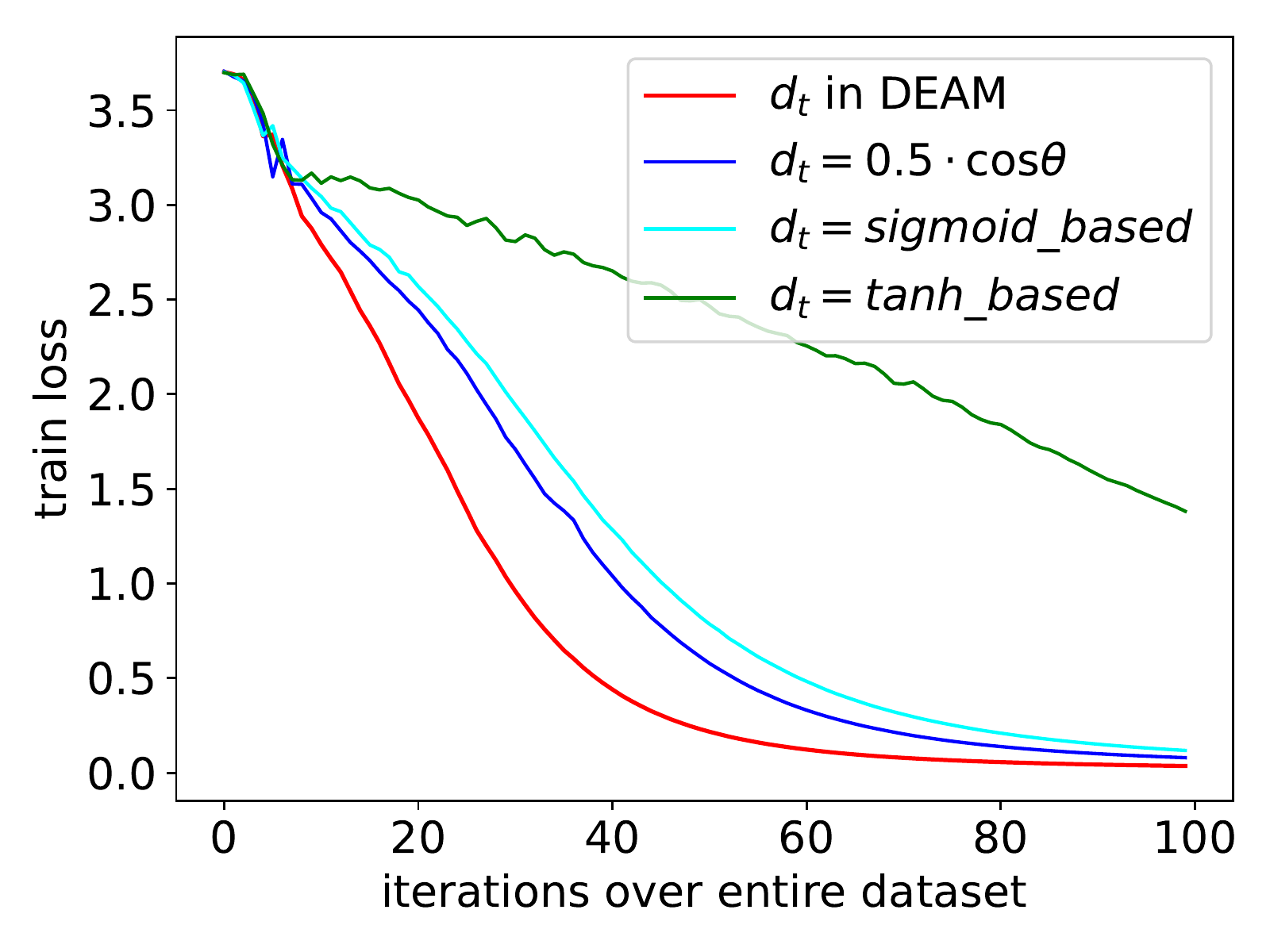}
	\caption{$d_t$ terms in other definitions}
	\vspace{-5pt}
	\label{fig:d_t_other}
\end{figure}
Due to the limited sapce, here we only exhibit the results on ORL dataset. From the results we can observe that the $d_t$ definnition in {\our} achieves the best congerging performance, which means our setting of $d_t$ in {\our} is effective.

\subsection{Time-consuming Analysis}

We have recorded the running time of {\our} and other comparison algorithms in every experiment, and list them in the Table~\ref{tab:running_time}. The running time shown in Table~\ref{tab:running_time} contains ``$>$'', which means the model still does not converge at the specific time. From the results we can observe that in all of our experiments, {\our} finally converges within the smallest mount of time. From the results in Figures~\ref{fig:dnn},\ref{fig:cnn} and Table~\ref{tab:running_time}, we can conclude that {\our} can converge not only in fewer epochs, but using less running time. The device we used is the Dell PowerEdge T630 Tower Server, with 80 cores 64-bit Intel Xeon CPU E5-2698 v4@2.2GHz. The total memory is 256 GB, with an extra (SSD) swap of 250 GB.

\section{Conclusion}\label{sec:conclusion}
In this paper, we have introduced a novel optimization algorithm, the {\our}, which implements the momentum with discriminative weights and the backtrack term. We have analyzed the advantages of the proposed algorithm and proved it by theoretical inference. Extensive experiments have shown that the proposed algorithm can converge faster than existing methods on both convex and non-convex situations, and the time consuming is better than existing methods. Not only the proposed algorithm can outperform other popular optimization algorithms, but fewer hyperparameters will be introduced, which makes the {\our} much more applicable. 


\newpage
\bibliographystyle{named}
\bibliography{reference}

\end{document}